\documentclass[letterpaper, 10 pt, conference]{ieeeconf}  

\IEEEoverridecommandlockouts                              

\usepackage{bm}
\usepackage{nicefrac}
\usepackage{amssymb}
\usepackage{amsmath}

\usepackage{amsthm}
\usepackage{cite}
\usepackage{xcolor}
\usepackage{graphicx}
\usepackage{url}
\usepackage{svg}
\usepackage{balance}
\newtheoremstyle{exampstyle}
  {3pt} 
  {3pt} 
  {\itshape} 
  {} 
  {\bfseries} 
  {.} 
  {.5em} 
  {} 

\theoremstyle{exampstyle} 
\newtheorem{definition}{Definition}
\newtheorem{lemma}{Lemma}
\newtheorem{theorem}{Theorem}

\newtheorem{assumption}{Assumption}
\newtheorem{problem}{Problem}

\theoremstyle{plain}

\definecolor{mumred}{RGB}{222,33,77}
\definecolor{mumgreen}{RGB}{0, 140, 0}
\definecolor{mumblue}{RGB}{0, 100, 222}
\definecolor{mumpurple}{RGB}{128, 0, 128}

\newcommand{\ofx}{\left(\bm{x}\right)}
\newcommand{\ofxt}{\left(\bm{x}_t\right)}
\newcommand{\ofxtau}{\left(\bm{x}_{\tau}\right)}
\newcommand{\dbetat}{\text{d}\bm{\beta}_t}
\newcommand{\dbetatau}{\text{d}\bm{\beta}_{\tau}}
\newcommand{\textmin}{$\mathrm{min}~$}
\newcommand{\textmax}{$\mathrm{max}~$}
\DeclareMathOperator*{\argmin}{arg\,min}

\title{\LARGE \bf
	Non-smooth Control Barrier Functions for Stochastic Dynamical Systems
}

\author{Matti Vahs and Jana Tumova
	\thanks{This work was partially supported by the Wallenberg AI, Autonomous
		Systems and Software Program (WASP) funded by the Knut and Alice
		Wallenberg Foundation. This research has been carried out as part of the Vinnova Competence Center for Trustworthy Edge Computing Systems and Applications at KTH Royal Institute of Technology.}
	\thanks{The authors are with the Division of Robotics, Perception and Learning, School of Electrical Engineering and Computer Science , KTH Royal Institute of Technology, Stockholm, Sweden and also affiliated with Digital Futures. Mail addresses: {\{\tt\small vahs, tumova\}}
		{\tt\small @kth.se}}%
}

\begin{document}
	\maketitle
	\thispagestyle{empty}
	\pagestyle{empty}

	\begin{abstract}
        Uncertainties arising in various control systems, such as robots that are subject to unknown disturbances or environmental variations, pose significant challenges for ensuring system safety, such as collision avoidance. At the same time, safety specifications are getting more and more complex, e.g., by composing multiple safety objectives through Boolean operators resulting in non-smooth descriptions of safe sets. Control Barrier Functions (CBFs) have emerged as a control technique to provably guarantee system safety. In most settings, they rely on an assumption of having deterministic dynamics and smooth safe sets. This paper relaxes these two assumptions by extending CBFs to encompass control systems with stochastic dynamics and safe sets defined by non-smooth functions. By explicitly considering the stochastic nature of system dynamics and accommodating complex safety specifications, our method enables the design of safe control strategies in uncertain and complex systems. We provide formal guarantees on the safety of the system by leveraging the theoretical foundations of stochastic CBFs and non-smooth safe sets. Numerical simulations demonstrate the effectiveness of the approach in various scenarios.
	\end{abstract}

	\section{INTRODUCTION}
    Control Barrier Functions (CBFs) provide a powerful framework for designing controllers that guarantee system safety by imposing constraints on the system's state variables. These constraints, which describe a set of safe states, act as virtual barriers preventing the system from entering undesirable regions in its state space \cite{ames2019control}. Many of the existing approaches in the CBF literature have focused on deterministic systems, assuming precise knowledge of the system dynamics \cite{ames2019control, grandia2021multi, xu2018safe, cortez2019control}. While these results have been instrumental in addressing safety concerns in such systems, real-world applications often involve stochastic dynamics, introducing uncertainty that cannot be ignored \cite{thrun2005probabilistic}. Moreover, the definition of safe sets has traditionally been restricted to smooth functions, limiting their applicability in scenarios where safe sets are defined by non-smooth functions, such as those constructed through Boolean compositions.
    
    This paper aims to bridge these gaps by extending the concept of Control Barrier Functions to settings with stochastic dynamics and safe sets defined by non-smooth functions. We present a novel approach that combines the theoretical foundations of stochastic control barrier functions (SCBFs) \cite{clark} and non-smooth safe sets, enabling the design of safe control strategies for complex systems operating in uncertain environments.

    To illustrate the importance and relevance of this work, let us consider a motivating example involving a quadrotor flying in an unstructured environment as shown in Fig.~\ref{fig:Intro}. The drone is equipped with sensors for estimating its state. However, due to unpredictable wind forces, the dynamics of the system become stochastic in nature. The wind forces introduce uncertainties that cannot be ignored in the control design process, making it crucial to account for the stochastic dynamics to ensure collision avoidance. A natural way of ensuring safety is to keep the distance to obstacles greater than zero. However, this distance function (captured e.g., by a Euclidean Signed Distance Field (ESDF)) is non-smooth due to multiple non-smooth obstacles, as illustrated in Fig.~\ref{fig:Intro}.
    The ESDF is a continuous function that is positive in the collision-free space but has a discontinuous gradient in certain regions of the state space.

    \begin{figure}[t]
    \centering
    \includegraphics[scale=1.4]{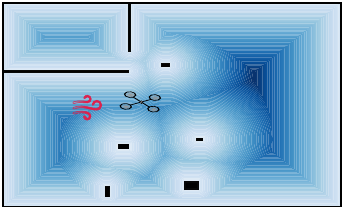}
    \caption{Illustration of a drone operating in a cluttered environment represented as a 2D occupancy grid map. The safe set is defined by an ESDF which is positive everywhere in the collision-free space. A stochastic wind disturbance acting on the drone is colored in red.}
    \label{fig:Intro}
    \vspace{-0.5cm}
\end{figure}
    Our contributions can be summarized as follows: we extend the theoretical analysis of Stochastic CBFs, proposed in \cite{clark}, to settings with non-smooth safe sets, offering a comprehensive solution for ensuring safety in such uncertain and complex systems. Specifically, we provide a control framework that
    \begin{enumerate}
        \item offers provable safety guarantees under stochastic dynamics as well as non-smooth safe sets, 
        \item enables users to encode desired system behavior as Boolean compositions
        \item is evaluated in various challenging simulation scenarios.
    \end{enumerate}
	\subsection{Related Work}
    CBFs have been explored extensively in the case of deterministic systems. They have been shown to work well even in a variety of real-world robotic scenarios such as legged locomotion \cite{grandia2021multi}, autonomous drone flight \cite{xu2018safe} or robotic grasping~\cite{cortez2019control}. Although safety has been ensured in these applications, precise knowledge about the state as well as the system dynamics are required which is often not given for robots in the wild. 
    Furthermore, one of the key assumptions in classic CBF formulations is that the safe set is described by a continuously differentiable function to ensure forward invariance \cite{ames2019control}. This assumption can be limiting with increasing complexity of safety specifications such as avoidance of non-smooth obstacles.

    To account for potential non-smoothnesses in safe sets, \cite{glotfelter2017nonsmooth} proposed non-smooth barrier functions (NBFs) for deterministic systems in which generalized gradients are leveraged to ensure forward invariance. These results have been used in several problems where non-smooth safe sets naturally occur such as network connectivity maintenance \cite{ong2021network} or visual coverage \cite{funada2020visual}. Further, CBFs have been designed in \cite{8404080} for non-smooth Signal Temporal Logic (STL) specifications by using smooth approximations.
    However, applying non-smooth CBFs to systems with uncertain dynamics cannot ensure safety due to the assumption of precise knowledge of the system dynamics.

    In \cite{clark} and \cite{santoyo2021barrier}, stochastic CBFs (SCBFs) have been proposed which provide safety guarantees almost surely for uncertain systems that are described by stochastic differential equations (SDEs). The foundation of SCBFs is based on the theory of martingales which we build upon in our work. Recently, risk-aware CBFs (RA-CBFs) \cite{black2023safety} have been introduced which provide probabilistic safety guarantees for stochastic systems which deviate from the martingale theory. In their work, the probability of leaving the safe set is bounded over a limited time horizon. Whilst stochastic CBF approaches have been leveraged in various settings such as safe learning \cite{pereira2021safe} and certification of neural network controllers \cite{mazouz2022safety}, they rely on twice continuously differentiable barrier functions to ensure safety. In \cite{singletary2022safe}, risk CBFs are defined for discrete-time stochastic dynamics and Boolean composed safe sets which, however, do not provide theoretical safety guarantees with probability one. 
    To the authors best knowledge, the combination of non-smooth safe sets and continuous-time stochastic systems has not been explored yet.

    As CBFs are often seen as the \emph{dual} of Control Lyapunov Functions (CLFs), progress in that field has to be considered as well. In \cite{florchinger1997feedback}, CLFs for stochastic systems are introduced and extended to non-smooth functions in \cite{ge2012non}. The theoretical results in \cite{ge2012non} are based on a partition of the state space into areas in which the CLF is smooth and the resulting controller is used for stabilization of quantum filters. In contrast to their control approach, we derive a general formulation expressed as a quadratic program (QP) that can be applied to various systems described by SDEs.
	\section{Preliminaries}
    Consider a general nonlinear SDE in It\^o form
    \begin{align}
    \mathrm{d}\bm{x}_t &= \bm{a}(\bm{x}_t) \text{d}t + \bm{b}(\bm{x}_t) \dbetat
    \label{eq:GeneralSDE}
    \end{align}
    where $\bm{a}\ofx \in \mathbb{R}^n$ is the drift term, $\bm{b}\ofx \in \mathbb{R}^{n \times \ell}$ is the diffusion term and $\bm{\beta}_t \in \mathbb{R}^{\ell}$ denotes $\ell$ dimensional Brownian motion. In this work, we assume that the drift term can be discontinuous as an effect of switching control laws. Next, we state assumptions which ensure that Eq.~\eqref{eq:GeneralSDE} has a solution.
    \begin{assumption}
        Suppose that
        \begin{enumerate}
            \item the drift term $\bm{a}\ofx$ is piecewise Lipschitz with its exceptional set $\Theta$ being a $C^3$ hypersurface.
            \item the diffusion term $\bm{b}\ofx$ is a non-degenerate diagonal matrix
            \item the diffusion term $\bm{b}\ofx$ is globally Lipschitz.
        \end{enumerate}
        \label{assmp:SDE}
    \end{assumption}
    The exceptional set $\Theta$ is an $n-1$ dimensional submanifold of $\mathbb{R}^n$ consisting of finitely many components which describes where the discontinuities of the drift term occur.
    These assumptions are needed to make statements about the existence and uniqueness of solutions to Eq.~\eqref{eq:GeneralSDE}.
    \begin{theorem}[Thm. 3.21, \cite{leobacher2017strong}]
        Under Assumption 1, the SDE in Eq.~\eqref{eq:GeneralSDE} has a unique global strong solution.
    \end{theorem}
    Note that our assumptions are stronger than the ones proposed in \cite{leobacher2017strong} since we assume that $\bm{b}$ is a non-degenerate diagonal matrix which is motivated by the application of robotics problems in which additive Gaussian noise is assumed on each state dimension. As a consequence of $\bm{b}$ being diagonal, it is ensured that the diffusion term always has an orthogonal component to $\Theta$. 

    Since a solution to Eq.~\eqref{eq:GeneralSDE} exists, we can apply It\^o's Lemma during a period in which the drift term is Lipschitz which is formally defined hereafter.
    \begin{definition}[It\^o's Lemma]
        Let $\varphi \ofx$ be a twice-differentiable function and $\bm{x}_t$ is a strong solution to Eq.~\eqref{eq:GeneralSDE} for all $t \in [t_i, t_{i+1}]$ in which the drift term is Lipschitz, then
        \begin{align*}
            \varphi\left(\bm{x}_t\right) &= \varphi\left(\bm{x}_{t_i}\right) &&+ \int_{t_i}^{t_{i+1}} \frac{\partial \varphi}{\partial \bm{x}} \bm{a}(\bm{x}) + \frac{1}{2} \mathrm{tr}\left[\bm{b}^T \frac{\partial^2 \varphi}{\partial \bm{x}^2} \bm{b}\right] \text{d} \tau\\
            &~ &&+ \int_{t_i}^{t_{i+1}} \bm{b} \frac{\partial \varphi}{\partial \bm{x}} \text{d}\bm{\beta}_{\tau}.
        \end{align*}
    \end{definition}

    \section{Problem Statement}
    We consider a stochastic system in control-affine form
    \begin{align}
        \mathrm{d}\bm{x}_t &= \left(\bm{f}\ofxt + \bm{g}\ofxt \bm{u}_t\right) \text{d}t + \bm{\sigma}\ofxt \dbetat
    \label{eq:SDE}
    \end{align}
    where $\bm{x} \in \mathcal{X} \subseteq \mathbb{R}^n$ is the state, $\bm{u} \in \mathcal{U} \subseteq \mathbb{R}^m$ is the control input and $\bm{\sigma} \in \mathbb{R}^{n\times \ell}$ is the diffusion term. We assume that the drift and diffusion terms satisfy Assumption~\ref{assmp:SDE}.

    We require our system to satisfy a safety constraint which is expressed as an invariance property where the state is supposed to remain in a safe set $\mathcal{C}$ at all times, i.e. $\bm{x}_t \in \mathcal{C} \hspace{0.2cm} \forall t\geq 0$.
    The safe sets we consider are defined as
    \begin{align}
    \mathcal{C} &=\left\{\bm{x} \in \mathcal{X} \mid h \ofx \geq 0\right\}\label{eq:safeset}\\
    \partial \mathcal{C} &=\left\{\bm{x} \in \mathcal{X} \mid h \ofx = 0\right\}\label{eq:safesetboundary}
    \end{align}
    where the safety function $h \ofx$ is potentially non-smooth, e.g. a function consisting of nested \textmin and \textmax operators. 
    \begin{problem}
    \label{prob:invariance}
        Given the stochastic system in Eq.~\eqref{eq:SDE} and a non-smooth safe set defined by \eqref{eq:safeset}-\eqref{eq:safesetboundary}, find a control input $\bm{u}$ that renders $\mathcal{C}$ forward invariant with probability one.
    \end{problem}
	
	\section{Non-smooth Stochastic CBFs}
    We present our solution to Problem \ref{prob:invariance} which is based on the construction of a CBF that is valid in the presence of stochastic uncertainties as well as non-smooth safe sets. To that end, we define a partition of the state space in which each partition defines an area with a smooth safety function. First, we show that forward invariance can be ensured during the evolution within one partition by using a modified version of the proof proposed in \cite{clark} and, second, we show that transitions between partitions do not affect the forward invariance property.
    \begin{definition}
        Let $h\left(\bm{x}\right)$ be a partition-based function of the form
        \begin{equation*}
            h\ofx = h_i \ofx, \hspace{0.5cm} \bm{x} \in \Phi_i, i \in \left\{1, 2, \dots, N\right\}
        \end{equation*}
        where
        \begin{enumerate}
            \item $h_i \colon \mathcal{X} \subseteq \mathbb{R}^n \mapsto \mathbb{R}$ is twice continuously differentiable for all $i \in \left\{1, 2,\dots,N\right\}$
            \item $\left\{\Phi_i\right\}_{i=1}^N$ is a partition of the state space, i.e. $\cup_{i=1}^N \Phi_i = \mathcal{X}$ and $\Phi_i \cap \Phi_j = \emptyset, \forall i \neq j$
            \item $h\ofx$ is continuous, i.e. $h_i\ofx = h_j \ofx \forall \bm{x} \in \partial \Phi_{ij}$ where $\partial \Phi_{ij}$ is the boundary between $\Phi_i$ and $\Phi_j$.
        \end{enumerate}
    \end{definition}
    An exemplary partition of a state space is illustrated in Fig.~\ref{fig:Example}. In order to ensure existence of a solution to the SDE in Eq.~\eqref{eq:SDE}, it is crucial that the state trajectory does not keep moving on the boundaries. 
    \begin{lemma}
    For a boundary $\partial \Phi_{ij}$ that is described by a function $p(\bm{x}) = 0$ it holds that for any $\varepsilon > 0$,
    $\mathrm{Pr}\left[p(\bm{x}_t)=0, \hspace{0.2cm}\forall t \in [t_0, t_0+\epsilon]\right] = 0$.
        \label{lemma:boundary}
    \end{lemma}
    \begin{proof}
        Let us consider the general SDE in Eq.~\eqref{eq:GeneralSDE} and an arbitrary boundary $\partial \Phi_{ij}$. Showing that $\mathrm{Pr}\left[p(\bm{x}_t)=0, \hspace{0.2cm}\forall t \in [t_0, t_0+\epsilon]\right] = 0$ for any $\epsilon>0$ concludes that the solution of (1) almost surely does not evolve on a boundary.

        By applying It\^o's Lemma in integral form, we get
        \begin{align*}
            p\left(\bm{x}_{t_0 + \epsilon}\right) - p\left(\bm{x}_{t_0}\right) &= \int_{t_0}^{t_0 + \epsilon} \frac{\partial p}{\partial \bm{x}} \bm{a}(\bm{x}) + \frac{1}{2} \mathrm{tr}\left[\bm{b}^T \frac{\partial^2 p}{\partial \bm{x}^2} \bm{b}\right] \text{d} \tau\\
            &+ \int_{t_0}^{t_0 + \epsilon} \bm{b} \frac{\partial p}{\partial \bm{x}} \text{d}\bm{\beta}_{\tau}
            \end{align*}
            where the left hand side is zero, as both states are assumed to be on the edge $\partial \Phi_{ij}$. The first integral on the right hand side is deterministic while the second one is a stochastic It\^o integral. Thus in order for the equation to hold, the It\^o integral needs to take a specific value. However, since any increment $\text{d}\bm{\beta}$ is a zero mean Gaussian random variable, the probability of taking on a deterministic value is zero and, thus, $\mathrm{Pr}\left[p(\bm{x}_t)=0, \hspace{0.2cm}\forall t \in [t_0, t_0+\epsilon]\right] = 0$.
    \end{proof}
    This Lemma is crucial for the design of non-smooth stochastic CBFs. In comparison with deterministic non-smooth CBFs \cite{glotfelter2017nonsmooth}, we do not need to consider the generalized gradient in non-smooth points because we always have an instantaneous switch between two smooth regions of the safe set. 
    Next, we introduce our non-smooth stochastic CBFs (NSCBFs) for safe control under stochastic uncertainties. 
    \begin{definition}
    \label{def:NSCBF}
        Given a safe set $\mathcal{C}$ defined by \eqref{eq:safeset}-\eqref{eq:safesetboundary}, the partition-based function $B\ofx$ serves as a NSCBF for the system \eqref{eq:SDE}, if $\forall \bm{x} \in \Phi_i, \forall i \in \left\{1, 2, \dots, N\right\}$ 
        \begin{enumerate}
            \item there exist class-K functions $\alpha_1$ and $\alpha_2$ such that 
            \begin{align}
                \label{eq:SCBF1}
                \frac{1}{\alpha_1\left(h_i\left(\bm{x}\right)\right)} \leq \frac{1}{B_i\ofx} \leq \frac{1}{\alpha_2\left(h_i\left(\bm{x}\right)\right)},
            \end{align}
            \item there exists a class-K function $\alpha_3$ and a $\bm{u} \in \mathcal{U}$ such that
            \begin{align*}
            \frac{\partial B_i}{\partial \bm{x}} \left(\bm{f}\ofx + \bm{g} \ofx \bm{u}\right) + \frac{1}{2} \mathrm{tr}\left[\bm{\sigma}^T \frac{\partial^2 B_i}{\partial \bm{x}^2} \bm{\sigma}\right]\\
            \leq \alpha_3(h_i\ofx).
        \end{align*}
        \end{enumerate}
    \end{definition}
    \begin{figure}[t]
    \centering
    \resizebox{1\linewidth}{!}{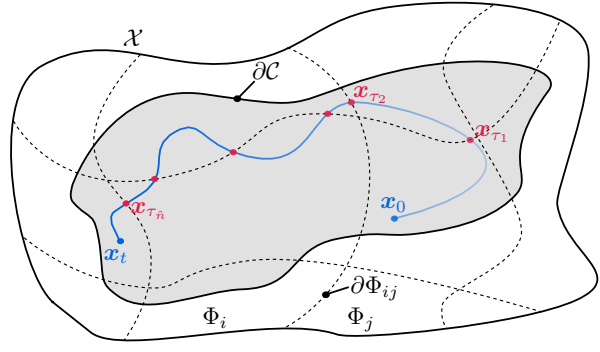}
    \caption{Illustration of the partition-based CBF in a two dimensional state space $\mathcal{X}$. The dashed lines indicate the edges between partitions $\Phi_i$ and $\Phi_j$. An exemplary trajectory satisfying the safety constraint is shown in blue while the states at exit times $\tau_k$ are colored in red.}
    \label{fig:Example}
    \vspace{-0.5cm}
\end{figure}
    In the smooth case \cite{clark}, stochastic CBFs have been proposed to ensure forward invariance. We extend those results to non-smooth domains which is shown subsequently. 

    Our proof of forward invariance for NSCBFs is structured in two parts. First, we show that forward invariance can be ensured if a system trajectory stays within one partition $\Phi_i$ of the state space in which the overall stochastic CBF is smooth. Second, we show that transitions in between partitions $\Phi_i$ and $\Phi_j$ do not affect the forward invariance of the safe set.
    
    Under Assumption \ref{assmp:SDE}, it is ensured that if the system trajectory $\bm{x}_t$ intersects the boundary $\partial \Phi_{ij}$ 
    it leaves the boundary instantly, almost surely, see Lemma \ref{lemma:boundary}. Thus, we can define exit times $\tau_i$ at which the system trajectory transitions from a partition $\Phi_i$ to $\Phi_j$ which yields a sequence of time points $\left\{\tau_i\right\}_{i=1,2,\dots}$ at which $\bm{x}$ crosses $\partial \Phi$, see Figure~\ref{fig:Example}.
    \begin{theorem}
        If a locally Lipschitz control input $\bm{u}_t$ satisfies Def. \ref{def:NSCBF} for a given safe set $\mathcal{C}$, then $\mathrm{Pr}\left[\bm{x}_t \in \mathcal{C} \cap \Phi_i, \hspace{0.2cm} \forall t \in [\tau_i, \tau_{i+1})\right]=1$, provided that $\bm{x}_{\tau_i} \in \mathcal{C} \cap \Phi_i$.
        \label{thm:ForwardInvSmooth}
    \end{theorem}
    \begin{proof}  The proof of the theorem is strongly inspired by the proof of Theorem 1 in \cite{clark}. Therein, the theorem is formulated for a smooth safe set and for all times $t$.
        We show that for any partition $\Phi_i$, any $t\in [\tau_i, \tau_{i+1})$, any $\varepsilon > 0$ and any $\delta \in (0, 1)$,
        \begin{equation*}
            \mathrm{Pr}\left[\underset{t' <t}{\mathrm{sup}}~B_i\left(\bm{x}_{t'}\right) > K \right] \leq \delta
        \end{equation*}
        for some $K>0$ which would prove forward invariance almost surely in a period where $\bm{x}$ evolves in $\Phi_i$. Let $\theta = B_i\left(\bm{x}_{\tau_i}\right)$ and construct $K$ such that
        \begin{align*}
            K > \frac{\theta + (t-\tau_i)\alpha_3\left(\alpha_2^{-1}\left(\frac{1}{\theta}\right)\right)}{\delta}.
        \end{align*}
        During one time period in which $h\ofx=h_i\ofx$ is smooth, we can apply It\^o's lemma to obtain
        \begin{align*}
            B_i\ofxt = &B_i\left(\bm{x}_{\tau_i}\right) + \int_{\tau_i}^t\frac{\partial B_i}{\partial \bm{x}} \left(\bm{f}\ofxtau + \bm{g} \ofxtau \bm{u}_{\tau}\right)\\
            &+ \frac{1}{2} \mathrm{tr}\left[\bm{\sigma}_{\tau}^T \frac{\partial^2 B_i}{\partial \bm{x}^2} \bm{\sigma}_{\tau}\right] \text{d}\tau + \int_{\tau_i}^t \bm{\sigma}_{\tau} \frac{\partial B_i}{\partial \bm{x}} \dbetatau.
        \end{align*}
        where $\bm{\sigma}_{\tau} = \bm{\sigma}\ofxtau$. Consider a sequence of stopping times\footnote{An alternative proof of a similar theorem is available in \cite{so2023almost} that is not based on stopping times.} $\eta_k$ and $\zeta_k$ at which $h_i(\bm{x})$ intersects with $\theta$, i.e.
        \begin{align*}
            &\eta_0 = \tau_i, \hspace{0.5cm} \zeta_0 = \mathrm{inf} \left\{ t \mid B_i(\bm{x}_t)< \theta\right\}\\
            &\eta_k = \mathrm{inf} \left\{ t \mid B_i(\bm{x}_t)> \theta \land t > \zeta_{k-1}\right\}\\
            &\zeta_k = \mathrm{inf} \left\{ t \mid B_i(\bm{x}_t)< \theta \land t > \eta_{k-1}\right\}.
        \end{align*}
        which is constructed such that $B_i \ofxt \geq \theta, ~ \forall t \in [\eta_k, \zeta_k]$ and, vice versa, $B_i \ofxt \leq \theta, ~ \forall t \in [\zeta_k, \eta_{k+1}]$ for $k=0,1,\dots$.
        Further, we define a random process
        \begin{equation}
            U_t = U_0 + \sum_{k=0}^{\infty} \int_{\eta_k \land t}^{\zeta_k \land t} \alpha_3\left(\alpha_2^{-1}\left(\frac{1}{\theta}\right)\right) \text{d}\tau
            + \int_{\eta_k \land t}^{\zeta_k \land t} \bm{\sigma} \frac{\partial B_i}{\partial \bm{x}} \text{d}\bm{\beta}_{\tau}\label{eq:randomProcess}
        \end{equation}
        where $U_0 = \theta$ and $a \land b$ denotes the minimum between $a$ and $b$. This random process is a \emph{submartingale}, i.e. it satisfies $\mathbb{E}\left[U_t \mid U_s\right] \geq U_s$ which has been proven in \cite{clark}.
        
        We will show by induction that $B_i(\bm{x}_t) \leq U_t$ for all $t \in[\tau_i, \tau_{i+1}]$. At time $t=\tau_i=\eta_0$, both statements hold as $B_i(\bm{x}_{\tau_i}) = \theta = U_{\tau_i}$ by construction of stopping times. The induction starts for the interval $t \in [\eta_k, \zeta_{k}]$ and is finished with the interval $[\zeta_k, \eta_{k+1}]$. For the former interval, we have
        \begin{align}
            B_i\left(\bm{x}_t\right) &= &&B_i\left(\bm{x}_{\eta_k}\right) + \int_{\eta_k}^t\frac{\partial B_i}{\partial \bm{x}} \left(\bm{f}(\bm{x}_{\tau}) + \bm{g} (\bm{x}_{\tau}) \bm{u}_{\tau}\right)\label{eq:comparison1}\\
            &~ &&+ \frac{1}{2} \mathrm{tr}\left[\bm{\sigma}^T \frac{\partial^2 B_i}{\partial \bm{x}^2} \bm{\sigma}\right] \text{d}\tau + \int_{\eta_k}^t \bm{\sigma} \frac{\partial B_i}{\partial \bm{x}} \text{d}\bm{\beta}_\tau\label{eq:comparison2}\\
            U_t &= &&U_{\eta_k} + \int_{\eta_k}^t \alpha_3\left(\alpha_2^{-1}\left(\frac{1}{\theta}\right)\right) \text{d}\tau + \int_{\eta_k}^t \bm{\sigma} \frac{\partial B_i}{\partial \bm{x}} \text{d}\bm{\beta}_\tau\label{eq:comparison3}
        \end{align}
        which we compare by terms. For the first term $B_i\left(\bm{x}_{\eta_k}\right) = \theta \leq U_{\eta_k}$ holds as main assumption of the induction process while the last terms are common. For the middle term, we know that a control input $\bm{u}$ satisfying Def. \ref{def:NSCBF} results in
        \begin{align*}
            \frac{\partial B_i}{\partial \bm{x}} \left(\bm{f}(\bm{x}_{\tau}) + \bm{g} (\bm{x}_{\tau}) \bm{u}_{\tau}\right)
            + \frac{1}{2} \mathrm{tr}\left[\bm{\sigma}_{\tau}^T \frac{\partial^2 B_i}{\partial \bm{x}^2} \bm{\sigma}_{\tau}\right] \\
            \leq \alpha_3\left(h_i(\bm{x}_{\tau})\right)
            \leq \alpha_3\left(\alpha_2^{-1}\left(\frac{1}{B_i\ofxtau}\right)\right)\leq \alpha_3\left(\alpha_2^{-1}\left(\frac{1}{\theta}\right)\right)
        \end{align*}
        which can be obtained by inverting Eq.~\eqref{eq:SCBF1}. Thus, we conclude by comparison of terms between \eqref{eq:comparison1}-\eqref{eq:comparison2} and \eqref{eq:comparison3} that $B(\bm{x}_t) \leq U_t$ for all $t \in [\eta_k, \zeta_{k}]$.

        In the next interval $[\zeta_k, \eta_{k+1}]$, we observe that the integrands in Eq. \eqref{eq:randomProcess} are only ``active'' in periods $[\eta_k, \zeta_{k}]$ and thus $U_t = U_{\zeta_k}$ is constant. Therefore, it can be easily concluded that $U_t = U_{\zeta_k} \geq \theta \geq B(\bm{x}_t)$ for all $t\in [\zeta_k, \eta_{k+1}]$ which completes the proof by induction.

        Using this results, it holds that
        \begin{align*}
            \mathrm{Pr}\left[\underset{t' <t}{\mathrm{sup}}~B_i\left(\bm{x}_{t'}\right) > K \right] \leq \mathrm{Pr}\left[\underset{t' <t}{\mathrm{sup}}~U_{t'} > K \right]
        \end{align*}
        where the latter term is further bounded by applying Doob's martingale inequality \cite{doob}, reading
        \begin{align*}
            \mathrm{Pr}\left[\underset{t' <t}{\mathrm{sup}}~U_{t'} > K \right] \leq \frac{\mathbb{E}\left[U_t^+\right]}{K}
        \end{align*}
        where $U_t^+ = \mathrm{max}\left\{U_t, 0\right\}$. The expectation of the random process $U_t$ is given by
        \begin{align}
            \mathbb{E}\left[U_t\right] = U_0 + \mathbb{E}\left[\sum_{k=0}^{\infty} \int_{\eta_k \land t}^{\zeta_k \land t} \alpha_3\left(\alpha_2^{-1}\left(\frac{1}{\theta}\right)\right) \text{d}\tau\right]\label{eq:expectationRandomProcess}
        \end{align}
        as the expectation of an integral over Brownian motion vanishes under Assumption \ref{assmp:SDE}. Since the second term in Eq.~\eqref{eq:expectationRandomProcess} is upper bounded by $\alpha_3\left(\alpha_2^{-1}\left(\frac{1}{\theta}\right)\right) (t - \tau_i)$, we have 
        \begin{align*}
            \mathrm{Pr}\left[\underset{t' <t}{\mathrm{sup}}~U_{t'} > K \right] 
            &\leq \frac{\theta + \alpha_3\left(\alpha_2^{-1}\left(\frac{1}{\theta}\right)\right) (t - \tau_i)}{K}\leq \delta
        \end{align*}
        which completes our proof.
    \end{proof}
        
    Next, we show that forward invariance can be guaranteed almost surely for any $t \geq 0$. 
    \begin{theorem}
        If a locally Lipschitz control input $\bm{u}_t$ satisfies Def.~\ref{def:NSCBF} for a given safe set $\mathcal{C}$, then $\mathrm{Pr}\left[\bm{x}_t \in \mathcal{C}, \hspace{0.2cm} \forall t\geq t_0\right]=1$, provided that $\bm{x}(t_0) \in \mathcal{C}$.
    \end{theorem}
    \begin{proof}
    We define a sequence of intervals $\left\{I_k\right\}_{k=0}^{\hat{n}(t)}$ with $I_k = [\tau_k, \tau_{k+1})$ where $\hat{n}(t)$ is the interval at the current time $t$. These intervals correspond the durations, the state spends in one of the partitions as shown in Fig.~\ref{fig:Example}. We will now show by induction that if $\mathrm{Pr}\left[\bm{x}_t \in \mathcal{C} \cap \Phi_k, \hspace{0.2cm} \forall t \in I_k\right]=1$, then it also holds true that $\mathrm{Pr}\left[\bm{x}_t \in \mathcal{C} \cap \Phi_{k+1}, \hspace{0.2cm} \forall t \in I_{k+1}\right]=1$.

    In the base case $k=0$, we know that for $\tau_0 = 0$ it holds that $h(\bm{x}_{\tau_0}) \geq 0$ since we always start within the safe set. Additionally we know by application of Thm.~\ref{thm:ForwardInvSmooth} that forward invariance of $\mathcal{C} \cap \Phi_0$ is ensured during the time interval $t \in [\tau_0, \tau_1)$ until the first exit time occurs as proven before.
    
    Now consider an arbitrary time interval $t \in I_k$ in which the system trajectory starts in the partition $\Phi_k$. Forward invariance is ensured in the previous time interval $I_{k-1}$ by assumption of the induction process and a global unique strong solution to \eqref{eq:SDE} exists, thus, due to the continuity of $h \ofx$, it follows that $h(\bm{x}_{\tau_k}) = h_{k-1}(\bm{x}_{\tau_k}) = h_{k}(\bm{x}_{\tau_k})$. Consequently, $h_k(\bm{x}_{\tau_k}) \geq 0$ and hence $\bm{x}_{\tau_k} \in \mathcal{C} \cap \Phi_k$. By application of Theorem \ref{thm:ForwardInvSmooth}, we know that there exists a control input $\bm{u}$ rendering $\mathcal{C} \cap \Phi_k$ forward invariant within the interval $I_k$ until the next exit time occurs. Hence by mathematical induction $\mathrm{Pr}\left[\bm{x}_t \in \mathcal{C} \cap \Phi_k, \hspace{0.2cm} \forall t \in I_k\right]=1$ is true for all positive integers $k$ which completes our proof.
    \end{proof}

    This concludes our theoretical analysis of forward invariance of non-smooth safe sets under stochastic dynamics. 

    \subsection{Control Synthesis}
    Although the theoretical analysis provides forward invariance results, it can happen that the resulting safety controller exhibits high frequency switching when the vector fields on either side of an edge $\partial \Phi_{ij}$ point in opposite directions which is not suitable for hardware implementations. In \cite{glotfelter2018boolean}, the almost-active set of functions $I_{\varepsilon}\ofx$ is introduced to ensure safety for smoothly composed deterministic Boolean non-smooth CBFs. We define a more general set as
    \begin{align}
        I_{\varepsilon}\ofx = \left\{i \in \mathbb{N} \mid \exists \partial \Phi_{ik} \land | h_i\ofx - h\ofx | \leq \varepsilon\right\}
    \end{align}
    which describes the set of barrier functions that are defined over neighboring partitions of the current active one, i.e. $\bm{x} \in \Phi_k$, and are close to the value of $h\ofx$ which occurs close to the boundary $\partial \Phi_{ik}$. 

    To synthesize safe control inputs under stochastic disturbances, we solve the quadratic program (QP)
    \begin{equation}
    \begin{aligned}
        \bm{u}^* = \argmin_{\bm{u} \in \mathcal{U}} \quad & \left(\bm{u} - \bm{u}_{\text{ref}}\right)^T \left(\bm{u} - \bm{u}_{\text{ref}}\right)\\
        \textrm{s.t.~~~} \quad & \frac{\partial B_i}{\partial \bm{x}} \left(\bm{f}\ofx + \bm{g} \ofx \bm{u}\right) +\frac{1}{2} \mathrm{tr}\left[\bm{\sigma}^T \frac{\partial^2 B_i}{\partial \bm{x}^2} \bm{\sigma}\right]\\
        ~\quad &  \leq \alpha_3\left(h_i\ofx\right), \hspace{0.5cm} \forall i \in I_{\varepsilon}\ofx.
    \end{aligned}
    \label{eq:QP}
    \end{equation}
    where $\bm{u}_{\text{ref}}$ is a reference control input and $h_i\ofx$ are the smooth barrier functions being active at the current state. This synthesis problem can be solved efficiently using off-the-shelf QP solvers. As a consequence of satisfying multiple barrier conditions, the control inputs generated on either side of a boundary do not vary as much which results in a smoother control signal which we further investigate in Sec.~\ref{sec:multiagent}. Furthermore, since the control input still satisfies Def.~\ref{def:NSCBF}, the derived forward invariance results are unaffected.

    
	\section{Simulation Study}
	We evaluate our NSCBFs in two different simulation scenarios in which non-smooth safe sets naturally occur, namely network connectivity constraints as well as a multi-agent collision avoidance task. We model the robots as stochastic single integrators of the form $\text{d}\bm{x} = \bm{u}~\text{d}t + \sigma \bm{I}~\text{d}\bm{\beta}\label{eq:agentdynamics}$
    where $\bm{x} = [p_x, p_y]^T$ is a 2D position, $\bm{u} \in \mathbb{R}^2$ is the control input and the variance is set to $\sigma = 0.025$. We pick $\alpha_1\ofx = \alpha_2\ofx = \alpha_3\ofx = \bm{x}$.

    \subsection{Single Agent with Boolean Safety Specification}
    In our first simulation scenario, we only consider a single agent which has a reference proportional controller navigating it from its initial position to a goal located at $\mathcal{G} = [1.8, 1]$. A safety specification is encoded as $(\mu_o\ofx \geq 0) \land \left(\mu_{n_1}\ofx \geq 0 \lor \mu_{n_2}\ofx \geq 0\right)$ where $\mu_o = \lVert \bm{x} - \mathcal{O} \rVert_2 - r_o$ is an obstacle avoidance constraint and $\mu_{n_1} = r_{n_1} - \lVert \bm{x} - \mathcal{N}_1 \rVert_2$ and $\mu_{n_2} = r_{n_2} - \lVert \bm{x} - \mathcal{N}_2 \rVert_2$ are static network connectivity constraints. The specific positions and radii are shown in Fig.~\ref{fig:Sim1}. Thus, the NSCBF can be constructed as
    \begin{align}
        h\ofx = \mathrm{min}\left(\mu_o\ofx, \mathrm{max}\left(\mu_{n_1}\ofx, \mu_{n_2}\ofx\right)\right).
    \end{align}
    By using $\min$ and $\max$ operators, it is ensured that $h\ofx \geq 0$ satisfies the Boolean safety specification. We simulate the proposed NSCBFs in Def.~\ref{def:NSCBF} in the Julia Programming language using the DifferentialEquations.jl package for solving SDEs \cite{rackauckas2017differentialequations}. To empirically validate that safety can be assured almost surely, we performed a monte carlo study in which simulated the scenario with the same initial conditions 500 times. The resulting trajectories with and without the proposed safety filter are shown in Fig.~\ref{fig:Sim1}. It can be seen that all trajectories of the robot satisfy both, the network connectivity constraints as well as the obstacle avoidance constraint. As a consequence of satisfying the safety constraint, the robot cannot reach its destination since it is outside the network coverage area. The average computation time for solving the QP took $t_c = 0.6 \pm 0.6$ milliseconds which shows that the controller can be easily run in real-time applications.
    \begin{figure}[t]
    \centering
    \includegraphics[scale=0.76]{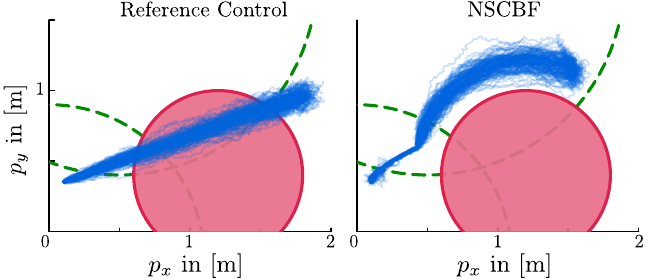}
    \caption{Simulation results for a Boolean safety specification. The network coverage areas are shown by green circles located at $\mathcal{N}_1 = [0, -0.2]^T$ and $\mathcal{N}_2 = [0.5, 1.8]^T$ with radii $r_1=1.1$ and $r_2=1.4$, respectively. The obstacle is located at $\mathcal{O}=[1.2, 0.4]^T$ with $r_o=0.6$ and illustrated as a red circle. The left plot shows the stochastic evolution of the system if only the reference proportional controller is used while the right plot depicts the behavior when using the proposed NSCBFs.}
    \label{fig:Sim1}
    \vspace{-0.7cm}
    \end{figure}
    \subsection{Multi Agent Collision Avoidance}
    \label{sec:multiagent}
    In our second simulation, we apply our method to a multi agent swapping task. In this scenario, $n_a$ agents are equidistantly distributed along the unit circle and their goal is to swap positions with the agent on the opposite side while avoiding collisions with all other agents. We formulate the safety specification as
    \begin{align}
        \bigwedge_{i=1}^{n_a} \bigwedge_{j = i+1}^{n_a} \left(\lVert \bm{x}_i - \bm{x}_j \rVert_2 \geq 2 r\right)
    \end{align}
    where $r$ denotes the collision radius of each agent. To synthesize control inputs, we define the joint state $\bar{\bm{x}} = \left[\bm{x}_1^T, \dots, \bm{x}_{n_a}^T\right]^T \in \mathbb{R}^{2n_a}$ as well as the joint dynamics. The NSCBF is constructed as
    \begin{align}
        h \ofx = \mathop{\min}_{\substack{i=1,\dots,n_a\\j=i,\dots, n_a}} \left(\lVert \bm{x}_i - \bm{x}_j \rVert_2 - 2r\right).
    \end{align}
    \begin{figure}[t]
    \centering
    \includegraphics[scale=0.7]{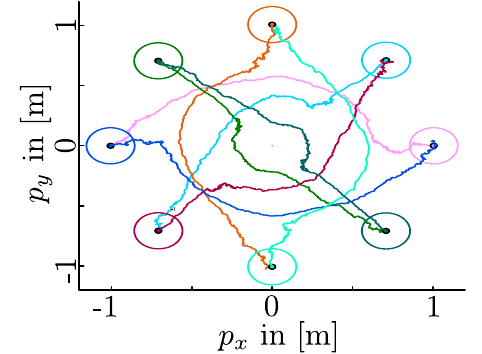}
    \caption{Simulation results for the multi agent swapping scenario. The agents initial positions are shown by the colored circles.}
    \label{fig:Sim2}
    \vspace{-0.6cm}
    \end{figure}
	Fig.~\ref{fig:Sim2} depicts the stochastic simulation of the system where it can be seen that all agents reach their destination while avoiding collisions with the other agents, thus satisfying the safety constraint. To emphasize the need for the almost-active set of functions used in control synthesis, we show the control inputs for $\varepsilon=0$ (only one barrier is active) and $\varepsilon=0.05$ in Fig.~\ref{fig:Sim3}. It can be observed that the almost-active set vastly reduces the effects of high frequency switching which makes the approach more applicable to real hardware.
	
	\section{CONCLUSIONS AND FUTURE WORK}
	Our work enables the design of safe controllers in the presence of stochastic disturbances on the system dynamics as well as non-smooth safe sets. The overall non-smooth safe set is decomposed into smooth partitions in which forward invariance can be guaranteed and it has been shown that transitions between partitions do not affect the forward invariance of the safe set. We have demonstrated our method in two challenging simulation scenarios of connectivity maintenance as well as multi agent collision avoidance and showcased that safety can be ensured at all times. In the future, we plan to use this theoretical result to apply NSCBFs in various settings where non-smoothnesses occur such as Signal Temporal Logic (STL) specifications or risk-aware safety specifications using particle filters.

	
	\bibliographystyle{IEEEtran}
	\bibliography{refs.bib}

	 \begin{figure}[t]
    \centering
    \includegraphics[scale=4]{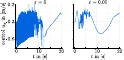}
    \caption{Exemplary control input $u_x$ of agent one over time for two different settings. On the left, only the barrier condition in the current partition is satisfied while the right figure shows the control signal for an almost-active set.}
    \label{fig:Sim3}
    \vspace{-0.6cm}
    \end{figure}
	
\end{document}